\documentclass[conference,12pt,onecolumn]{ieeeconf}

\newtheorem{assumption}{Assumption}
\newtheorem{lemma}{Lemma}
\newtheorem{proposition}{Proposition}
\newtheorem{theorem}{Theorem}

\IEEEoverridecommandlockouts                              

\usepackage{graphicx} 

\title{The Reinforce Policy Gradient Algorithm Revisited}
\author{Shalabh Bhatnagar\\
Department of Computer Science and Automation, \\Indian Institute of Science, \\Bengaluru 560012, India\\
        {\tt\small shalabh@iisc.ac.in}
\thanks{The author was supported by a J.~C.~Bose Fellowship, Project No.~DFTM/ 02/ 3125/M/04/AIR-04 from DRDO under DIA-RCOE, a project from DST-ICPS, and the RBCCPS, IISc.}}
\date{October 2023}

\begin{document}
\baselineskip=21pt
\maketitle

\begin{abstract}
We revisit the Reinforce policy gradient algorithm from the literature. Note that this algorithm typically works with cost returns obtained over random length episodes obtained from either termination upon reaching a goal state (as with episodic tasks) or from instants of visit to a prescribed recurrent state (in the case of continuing tasks). We propose a major enhancement to the basic algorithm. We estimate the policy gradient using a function measurement over a perturbed parameter by appealing to a class of random search approaches. This has advantages in the case of systems with infinite state and action spaces as it relax some of the regularity requirements that would otherwise be needed for proving convergence of the Reinforce algorithm. 
Nonetheless, we observe that even though we estimate the gradient of the performance objective using the performance objective itself (and not via the sample gradient), the algorithm converges to a neighborhood of a local minimum. 
We also provide a proof of convergence for this new algorithm.
\end{abstract}


\section{Introduction}
\label{introduction}

Policy gradient methods \cite{suttonbarto, sutton1999policy} are a popular class of approaches in reinforcement learning. The policy in these approaches is considered parameterized and one updates the policy parameter along a gradient search direction where the gradient is of a performance objective, normally the value function. The policy gradient theorem 
\cite{sutton1999policy, marbach2001simulation, cao} which is a fundamental result in these approaches relies on an interchange of the gradient and expectation operators and in such cases turns out to be the expectation of the gradient of noisy performance functions much like the previously studied perturbation analysis based sensitivity approaches for optimization via simulation \cite{hocao, choram1}. 

The Reinforce algorithm \cite{williams1992simple, suttonbarto} is a noisy gradient scheme for which the expectation of the gradient is the policy gradient, i.e., the gradient of the expected objective w.r.t.~the policy parameters. The updates of the policy parameter are however obtained once after the full return on an episode has been found. Actor-critic algorithms \cite{suttonbarto,konda1999actor,konda2003onactor, bhatnagar2009natural,bhatnagar2007incremental} have been presented in the literature as alternatives to the Reinforce algorithm as they perform incremental parameter updates at every instant but do so using two-timescale stochastic approximation algorithms. 


In this paper, we revisit the Reinforce algorithm and present a new algorithm for the case of episodic tasks or the stochastic shortest path setting. Our algorithm performs parameter updates upon termination of episodes, that is when goal or terminal states are reached. 
In this setting, as mentioned above, updates are performed only at instants of visit to a prescribed recurrent state \cite{cao, marbach2001simulation}. 
This algorithm is based on a single function measurement or simulation at a perturbed parameter value where the perturbations are obtained using independent Gaussian random variates. 

Gradient estimation in this algorithm is performed using the smoothed functional (SF) technique for gradient estimation \cite{rubinstein, bhatbor3, bhat2, bhatnagar-book}. The basic problem in this setting is the following: Given an objective function $J:\mathcal{R}^d\rightarrow\mathcal{R}$ such that $J(\theta) = E_\xi[h(\theta,\xi)]$, where $\theta\in \mathcal{R}^d$ is the parameter to be tuned and $\xi$ is the noise element, the goal is to find $\theta^*\in\mathcal{R}^d$ such that
\[
J(\theta^*) = \min_{\theta\in \mathcal{R}^d} J(\theta).
\]
Since the objective function $J(\cdot)$ can be highly nonlinear, one often settles for a lesser goal -- that of finding a local instead of a global minimum. 
In this setting, the Kiefer-Wolfowitz \cite{kw} finite difference estimates for the gradient of $J$ would correspond to the following: For $i=1,\ldots,d$,
\[
\nabla_i J(\theta) = \lim_{0<\delta\downarrow 0}
E_\xi\left[\frac{h(\theta+\delta e_i,\xi^+_i) - h(\theta-\delta e_i,\xi^-_i)}{2\delta}\right],
\]
where $e_i$ is the unit vector with 1 in the $i$th place and $0$ elsewhere. Further, $\xi^+_i$ and $\xi^-_i$, $i=1,\ldots,d$ are independent noise random variables having a common distribution. The expectation above is taken w.r.t this common distribution on the noise random variables. 
This approach does not perform well in practice when $d$ is large, since one needs $2d$ function measurements or simulations for a  parameter update. 

Random search methods such as simultaenous perturbation stochastic approximation (SPSA) \cite{spall1992multivariate, spall1997one, bhat1}, smoothed functional (SF) \cite{katkul, bhatbor3, bhat2} or random directions stochastic approximation (RDSA) \cite{kushcla, prashanth2017rdsa} typically require much less number of simulations.  
For instance, the gradient based algorithms in these approaches require only one or two simulations regardless of the parameter dimension $d$, while their Newton-based counterparts usually involve one to four system simulations for any parameter update (again regardless of $d$). A textbook treatment of random search approaches for stochastic optimization is available in \cite{bhatnagar-book}. 

We consider here a one-simulation SF algorithm where the gradient of $J(\theta)$ is estimated using a noisy function measurement, at the parameter $\theta+\delta\Delta$, where $\delta>0$ and $\Delta\stackrel{\triangle}{=}(\Delta_1,\ldots,\Delta_d)^T$ with each $\Delta_i \sim N(0,1)$ with $\Delta_i$ being independent of $\Delta_j$, $\forall j\not=i$. Further, $\Delta$ is independent of the measurement noise as well. The gradient estimate in this setting is the following: For $i=1,\ldots,d$,
\[
\nabla_i J(\theta) = \lim_{0<\delta\downarrow 0}
E_{\xi,\Delta}\left[\Delta_i \frac{h(\theta+\delta\Delta,\xi)}{\delta}\right].
\]
In the above, $\xi$ denotes the measurement noise random variable. 
The expectation above is with respect to the joint distribution of $\xi$ and $\Delta$. 
Before we proceed further, we present the basic Markov decision process (MDP)  setting as well as recall the Reinforce algorithm that we consider for the episodic setting. We remark here that there are not many analyses of Reinforce type algorithms in the literature in the episodic or stochastic shortest path setting. 


\section{The Basic MDP Setting}

By a Markov decision process, we mean a controlled stochastic process $\{X_n\}$ whose evolution is governed by an associated control-valued sequence $\{Z_n\}$. It is assumed that $X_n,n\geq 0$ takes values in a set $S$ called the state-space. Let $A(s)$ be the set of feasible actions in state $s\in S$ and ${\displaystyle A \stackrel{\triangle}{=} \cup_{s\in S} A(s)}$ denote the set of all actions. When the state is say $s$ and a feasible action $a$ is chosen, the next state seen is $s'$ with a probability $p(s'|s,a) \stackrel{\triangle}{=} P(X_{n+1}=s'\mid X_n=s, Z_n=a)$, $\forall n$. We assume these probabilities do not depend on $n$.
Such a process satisfies the controlled Markov property, i.e.,
\[ P(X_{n+1}=s' \mid X_n, Z_n, \ldots, X_0, Z_0) = p(s'\mid X_n, Z_n)
\mbox{ a.s.}
\]

By an admissible policy or simply a policy, we mean a sequence of functions
$\pi=\{\mu_0,\mu_1,\mu_2,\ldots\}$, with $\mu_i:S\rightarrow A$, $i\geq 0$, such that $\mu_i(s) \in A(s)$, $\forall s\in S$. The policy $\pi$ is a decision rule which specifies that if at instant $k$, the state is $i$, then the action chosen under $\pi$ by the decision maker would be $\mu_k(i)$.
A stationary policy $\pi$ is one for which $\mu_k = \mu_l \stackrel{\triangle}{=} \mu$, $\forall k,l=0,1,\ldots$. In other words, under a stationary policy, the function that decides the action-choice in a given state does not depend on the  instant $n$ at which the action is chosen. 

Associated with any transition to a state $s'$ from a state $s$ under action $a$, is a `single-stage' cost $g(s,a,s')$ where $g:S\times A\times S\rightarrow \mathcal{R}$ is called the cost function. The goal of the decision maker is to select actions $a_k, k\geq 0$ in response to the system states $s_k,k\geq 0$, observed one at a time, so as to minimize a long-term cost objective. We assume here that the number of states and actions is finite. 

\subsection{The Episodic or Stochastic Shortest Path Setting}

We consider here the episodic or the stochastic shortest path problem where decision making terminates once a goal or terminal state is reached. We let $1,\ldots,p$ denote the set of non-terminal or regular states and $t$ be the terminal state. Thus, $S=\{1,2,\ldots,p,t\}$ denotes the state space for this problem.

Our basic setting here is similar to Chapter 3 of  \cite{bertsekas2}, where it is assumed that under any policy there is a positive probability of hitting the goal state $t$ in at most $p$ steps starting from any initial (non-terminal) state, that would in turn signify that the problem would terminate in a finite though random amount of time.

Under a given policy $\pi$, define
\begin{equation}
    \label{vpi1}
V_\pi(s) = E_\pi\left[\sum_{k=0}^{T} g(X_k,\mu_k(X_k),X_{k+1})\mid X_0=s\right],
\end{equation}
where $T>0$ is a finite random time at which the process enters the terminal state $t$. Here $E_\pi[\cdot]$ indicates that all actions are chosen according to policy $\pi$ depending on the system state at any instant. We assume that there is no action that is feasible in the state $t$ and so the process terminates once it reaches $t$.

Let $\Pi$ denote the set of all admissible policies. The goal here is to find the optimal value function $V^*(i), i\in S$, where
\begin{equation}
    \label{vstar1}
V^*(i) = \min_{\pi\in \Pi} V_\pi(i) = V_{\pi^*}(i), \mbox{ } i\in S.
\end{equation}
Here $\pi^*$ denotes the optimal policy, i.e., the one that minimizes $V_\pi(i)$ over all policies $\pi$. A related goal then would be to find the optimal policy $\pi^*$. It turns out that in these problems, there exist stationary policies that are optimal, and so it is sufficient to search for an optimal policy within the class of stationary policies.

A stationary policy $\pi$ is called a proper policy (cf.~pp.174 of \cite{bertsekas2}) if
\[\hat{p}_\pi \stackrel{\triangle}{=} \max_{s=1,\ldots,p}P(X_p \not= t\mid X_0=s, \pi) <1.\]
In other words, regardless of the initial state $i$, there is a positive probability of termination after at most $p$ stages when using a proper policy $\pi$ and moreover $P(T<\infty) =1$ under such a policy.

An admissible policy (and so also a stationary policy) can be randomized as well. A randomized admissible policy or simply a randomized policy is a sequence of distributions $\psi =\{\phi_0,\phi_1,\ldots\}$ with each $\phi_i:S\rightarrow P(A)$. In other words, given a state $s$, a randomized
policy would provide a distribution $\phi_i(s) = (\phi_i(s,a), a\in A(s))$
for the action to be chosen in the $i$th stage. A stationary randomized policy is one for which $\phi_j=\phi_k\stackrel{\triangle}{=} \phi$, $\forall j,k=0,1,\ldots$. In this case, we simply call $\phi$ to be a stationary randomized policy. 
Here and in the rest of the paper, we shall assume that the policies are stationary randomized and are parameterized via a certain parameter $\theta \in C \subset \mathcal{R}^d$, a compact and convex set.

We make the following assumption:
\begin{assumption}
\label{properpolicy}
All stationary randomized policies $\phi_\theta$ parameterized by $\theta\in C$ are proper.
\end{assumption}
In practice, one might be able to relax this assumption (as with the model-based analysis of \cite{bertsekas2}) by (a) assuming that for policies that are not proper, $V_\pi(i) =\infty$ for at least one non-terminal state $i$ and (b) there exists a proper policy. 
The optimal value function satisfies in this case the following Bellman equation: For $s=1,\ldots,p$,
\begin{equation}
    \label{bellmaneq}
    V^*(s) = \min_{a\in A(s)} \left(\bar{g}(s,a) + \sum_{j=1}^{p}p(j\mid s,a) V^*(j)\right),
\end{equation}
where ${\displaystyle \bar{g}(s,a) = \sum_{j=1}^{p}p(j|s,a) g(s,a,j)+ p(t|s,a)g(s,a,t)}$ is the expected single-stage cost in a non-terminal state $s$ when a feasible action $a$ is chosen.
It can be shown, see \cite{bertsekas2}, that an optimal stationary proper policy exists.




\subsection{The Policy Gradient Theorem}

 Policy gradient methods perform a gradient search within the prescribed class of parameterized policies. 
 Let $\phi_\theta(s,a)$ denote the probability of selecting action $a\in A(s)$ when the state is $s\in S$ and the policy parameter is $\theta\in C$.
We assume that $\phi_\theta(s,a)$ is continuously differentiable in $\theta$. 
A common example here is of the parameterized Boltzmann or softmax policies. Let $\phi_\theta(s) \stackrel{\triangle}{=} (\phi_\theta(s,a),a\in A(s))$, $s\in S$ and $\phi_\theta \stackrel{\triangle}{=}(\phi_\theta(s), s\in S)$. 

We assume that trajectories of states and actions are available either as real data or from a simulation device. 
Let ${\displaystyle G_k = \sum_{j=k}^{T-1} g_j}$ denote the sum of costs until termination (likely when a goal state is reached) on a trajectory starting from instant $k$.
Note that if all actions are chosen according to a policy $\phi$, then the value and Q-value functions (under $\phi$) would be
    \(V_\phi(s) = E_\phi[G_k \mid X_k=s]\) and 
\(    Q_\phi(s,a) = E_\phi[G_k\mid X_k=s, Z_k=a],
\)
respectively. 
In what follows, for ease of notation, we let $V_\theta \equiv V_{\phi_\theta}$. 

The policy gradient theorem for episodic problems has the following form, cf.~\cite{sutton1999policy, suttonbarto}:
\begin{equation}
\label{pgt}
\nabla V_{\theta}(s_0) = \sum_{s\in S}\mu(s) 
\sum_{a\in A(s)} \nabla_\theta \pi(s,a) Q_\theta(s,a),
\end{equation}
where $\mu(s), s\in S$ is defined as ${\displaystyle \mu(s) = \frac{\eta(s)}{\sum_{s'\in S} \eta(s')}}$ where ${\displaystyle \eta(s) = \sum_{k=0}^{\infty} p^k(s|s_0,\pi_\theta)}$, $s\in S$, with $p^k(s|s_0,\pi_\theta)$ being the $k$-step transition probability of going to state $s$ from $s_0$ under the policy $\pi_\theta$.   A similar result holds for the long-run average cost setting with $\mu(s) = d^\pi(s)$ (the stationary distribution of $\{X_n\}$ under policy $\pi$), and $Q_\pi(s,a)$ is the state-action differential value function. Further, in the discounted cost setting too, a similar result holds but with $\eta(s)$ replaced by ${\displaystyle \eta(s) = \sum_{k=0}^{\infty} \gamma^k p^k(s|s_0,\pi_\theta)}$, where $0<\gamma<1$ is the discount factor. Proving the policy gradient theorem when the state-action spaces are finite is relatively straight forward \cite{sutton1999policy, suttonbarto}. However, one would require strong regularity assumptions on the system dynamics and performance function as with IPA or likelihood ratio approaches \cite{hocao} if the state-action spaces are non-finite i.e., countably infinite or continuously-valued sets.

The Reinforce algorithm \cite{suttonbarto, williams1992simple} makes use of the policy gradient theorem as the latter indicates that the gradient of the value function is the expectation of the gradient of a function of the noisy returns obtained from episodes. 
In what follows, we present an alternative algorithm based on Reinforce that incorporates a one-measurement SF gradient estimator. Our algorithm does not incorporate the policy gradient theorem and thus does not rely on an interchange between the gradient and expectation operators. Our algorithm incorporates a zeroth-order gradient approximation using the smoothed functional method and yet, like the Reinforce algorithm, requires only one sample trajectory. This is thus a clear advantage with our algorithm.
However, since our algorithm caters to episodic tasks, it performs updates at the instants of visit to a certain prescribed recurrent state as considered in \cite{cao, marbach2001simulation}. It is important to mention that such instants can be highly sparse in practice since in most practical systems, the number of state-action spaces can be very large in size.
We refer to our algorithm as the SF Reinforce algorithm. 


\section{The SF Reinforce Algorithm}

We consider here the case of episodic problems and the model-free setting whereby we do not assume any knowledge of the system model, i.e., the transition probabilities $p(s'\mid s,a)$, and in their place, we assume that we have access to data (either real or simulated). The data that is available is over trajectories of states, actions, single-stage costs and next states until termination.
We assume that multiple trajectories of data can be made available and the data on the $m$th  trajectory can be represented in the form of the tuples $(s^m_k,a^m_k,g^m_k,s^m_{k+1})$, $k=0,1,\ldots,T_m$ with $T_m$ being the termination instant on the $m$th trajectory, $m\geq 1$. Also, $s^m_j$ is the state at instant $j$, $j=k,k+1$ in the $m$th trajectory. Further, $a^m_k$ and $g^m_k$ are the action chosen and the cost incurred, respectively, at instant $k$ in the $m$th trajectory.
As mentioned before, we consider a class of stationary randomized policies that are parameterized by $\theta$ and satisfy Assumption~\ref{properpolicy}.


Let $\Gamma:{\cal R}^d\rightarrow C$
denote a projection operator that projects any $x=(x_1,\ldots,x_d)^T \in {\cal R}^d$ to its nearest point
in $C$. Thus, if $x\in C$, then $\Gamma(x)\in C$ as well.
For ease of exposition, we assume that $C$ is a
$d$-dimensional rectangle having the
form ${\displaystyle C = \prod_{i=1}^{d} [a_{i,\min}, a_{i,\max}]}$, where
$-\infty< a_{i,\min} < a_{i,\max} <\infty$, $\forall i=1,\ldots,d$. Then
$\Gamma(x)=(\Gamma_1(x_1),\ldots,\Gamma_d(x_d))^T$
with $\Gamma_i:{\cal R}\rightarrow [a_{i,\min},a_{i,\max}]$
such that
${\displaystyle \Gamma_i(x_i) = \min(a_{i,\max}, \max(a_{i,\min},x))}$, $i=1,\ldots,d$.
Also, let ${\cal C}(C)$ denote the space of all continuous functions from $C$ to ${\cal R}^d$.

In what follows, we present a procedure that incrementally updates the parameter $\theta$. Let $\theta(n)$ denote the parameter value obtained after the $n$th update of this procedure which depends on the $n$th episode and which is run using the policy parameter  $\Gamma(\theta(n)+\delta_n \Delta(n))$, for $n\geq 0$, where $\theta(n)=(\theta_1(n),\ldots,\theta_d(n))^T \in \mathcal{R}^d$, $\delta_n>0$ $\forall n$ with $\delta_n\rightarrow 0$ as $n\rightarrow\infty$ and $\Delta(n)=(\Delta_1(n),\ldots,\Delta_d(n))^T, n\geq 0$, where $\Delta_i(n),i=1,\ldots,d,n\geq 0$ are independent random variables distributed according to the $N(0,1)$ distribution. 

Algorithm (\ref{reinforce}) below is used to update the parameter $\theta \in C \subset \mathcal{R}^d$.
Let $\chi^n$ denote the $n$th state-action trajectory
$\chi^{n} = \{s^{n}_0,a^{n}_0,s^{n}_1,a^{n}_1,\ldots,s^{n}_{T-1},a^{n}_{T-1},s^{n}_T\}$, $n\geq 0$
where the actions $a^n_0,\ldots,a^n_{T-1}$ in $\chi^n$ are obtained using the policy parameter $\theta(n)+\delta_n\Delta(n)$. The instant $T$ denotes the termination instant in the trajectory $\chi^n$ and corresponds to the instant when the terminal or goal state $t$ is reached..
Note that the various actions in the trajectory $\chi^n$ are chosen according to the policy $\phi_{(\theta(n)+\delta_n\Delta(n))}$. The initial state is assumed to be sampled from a given initial distribution $\nu = (\nu(i),i\in S)$ over states.

Let ${\displaystyle G^n = \sum_{k=0}^{T-1} g^n_k}$ denote the sum of costs until termination on the trajectory $\chi^n$, with $g^n_k \equiv g(X^n_k,Z^n_k,X^n_{k+1})$ and where the superscript $n$ on the various quantities indicates that these correspond to the $n$th episode.
The update rule that we consider here is the following: For $n\geq 0, i=1,\ldots,d$,
\begin{equation}
    \label{reinforce}
    \theta_i(n+1) = \Gamma_i\left(\theta_i(n) -a(n) \left( \Delta_i(n) \frac{G^{n}}{\delta_n} \right)\right).
    \end{equation}
We assume that the step-sizes $a(n), n\geq 0$ satisfy the following requirement:
\begin{assumption}
\label{assum:ss}
The step-size sequence $\{a(n)\}$ satisfies $a(n)>0$, $\forall n$. Further,
\[\sum_n a(n)=\infty, \mbox{ } \sum_n \left(\frac{a(n)}{\delta_n}\right)^2 <\infty.
\]
\end{assumption}

Once the $n$th update of the parameter (i.e., $\theta(n)$) is obtained, the perturbed parameter $\theta(n)+\delta_n\Delta(n)$ is obtained after sampling $\Delta(n)$ from the multivariate Gaussian distribution as explained previously and thereafter a new trajectory governed by this perturbed parameter is generated with the initial state in each episode sampled according to a given distribution $\nu$. This is because each episode ends in the terminal state $t$ and any fresh episode starts from a non-terminal initial state sampled from the distribution $\nu$.



\section{Convergence Analysis}

We begin by rewriting the recursion (\ref{reinforce}) as follows:
\begin{equation}
    \label{reinforce2}
    \theta_i(n+1) = \Gamma_i\left(\theta_i(n) -a(n) E\left[\Delta_i(n)\frac{G^n}{\delta_n} | \mathcal{F}_n\right]
    + M^i_{n+1}\right),
    \end{equation}
    where
    \(M^i_{n+1} = \Delta_i(n) \frac{G^n}{\delta_n}
    - E\left[\Delta_i(n)\frac{G^n}{\delta_n} | \mathcal{F}_n\right], n\geq 0.
    \)
    Here, we let
    $\mathcal{F}_n \stackrel{\triangle}{=} \sigma(\theta(m), m\leq n, \Delta(m), \chi^m, m<n), n\geq 1$ as a sequence of increasing sigma fields with $\mathcal{F}_0 = \sigma(\theta(0))$. Let
    $M_n\stackrel{\triangle}{=} (M^1_n,\ldots,M^d_n)^T$, $n\geq 0$. 

\begin{lemma}
\label{firstlem}
$(M_n,\mathcal{F}_n),n\geq 0$ is a martingale difference sequence.
\end{lemma}
\begin{proof}
Notice that
\[M^i_{n} = \Delta_i(n-1) \frac{G^{n-1}}{\delta_{n-1}}
    - E\left[\Delta_i(n-1)\frac{G^{n-1}}{\delta_{n-1}} \mid \mathcal{F}_{n-1}\right].
    \]
    
    The first term on the RHS above is clearly measurable $\mathcal{F}_n$ while the second term is measurable $\mathcal{F}_{n-1}$ and hence measurable $\mathcal{F}_n$ as well.
Further, from Assumption~\ref{properpolicy}, each $M_n$
is integrable. Finally, it is easy to verify that
\[E[M^i_{n+1}\mid \mathcal{F}_n] =0, \mbox{ } \forall i.
\]
The claim follows.
\end{proof}


\begin{proposition}
\label{propest}
We have
\[E\left[\Delta_i(n) \frac{G^n}{\delta_n} \mid \mathcal{F}_n\right]
= \sum_{s\in S} \nu(s)\nabla_i V_{\theta(n)}(s) + o(\delta_n) \mbox{ a.s.}\]
\end{proposition}

\begin{proof}
Note that
\[ E\left[\Delta_i(n) \frac{G^n}{\delta_n} \mid \mathcal{F}_n\right]
= E\left[E\left[ \Delta_i(n) \frac{G^n}{\delta_n} \mid \mathcal{G}_n\right]\mid \mathcal{F}_n \right],
\]
where $\mathcal{G}_n \stackrel{\triangle}{=} \sigma(\theta(m), \Delta(m), m\leq n, \chi^m, m<n), n\geq 1$ is a sequence of increasing sigma fields with $\mathcal{G}_0 = \sigma(\theta(0),\Delta(0))$. It is clear that
$\mathcal{F}_n \subset \mathcal{G}_n, \forall n\geq 0$.
Now,
\[E\left[\Delta_i(n) \frac{G^n}{\delta_n} \mid \mathcal{G}_n\right]
= \frac{\Delta_i(n)}{\delta_n}
E[G^n\mid \mathcal{G}_n].
\]
Let $s^n_0=s$ denote the initial state in the trajectory $\chi^n$. Recall that the initial state $s$ is chosen randomly from the distribution $\nu$. Thus,
\[E[G^n\mid\mathcal{G}_n] = \sum_s \nu(s) E[G^n\mid s^n_0=s, \phi_{\theta(n)+\delta_n\Delta(n)}]\]
\[
= \sum_s \nu(s) V_{\theta(n)+\delta_n\Delta(n)}(s). \]
Thus, with probability one,
\[
E\left[\Delta_i(n) \frac{G^n}{\delta_n} \mid \mathcal{G}_n\right]
= \sum_{s} \nu(s)\left(\Delta_i(n) \frac{V_{\theta(n)+\delta_n\Delta(n)}(s)}{\delta_n}\right).
\]
Hence, it follows almost surely that
\[ E\left[\Delta_i(n) \frac{G^n}{\delta_n} \mid \mathcal{F}_n\right]
= \sum_s \nu(s) E\left[\Delta_i(n) \frac{V_{\theta(n)+\delta_n\Delta(n)}(s)}{\delta_n} \mid \mathcal{F}_n\right].
\]

Using a Taylor's expansion of $V_{\theta(n)+\delta_n\Delta(n)}(s)$ around $\theta(n)$ gives us
\[V_{\theta(n)+\delta_n\Delta(n)}(s_n) = V_{\theta(n)}(s_n) + \delta_n \Delta(n)^T \nabla V_{\theta(n)}(s_n)\]
\[+ \frac{\delta_n^2}{2} \Delta(n)^T \nabla^2 V_{\theta(n)}(s_n) \Delta(n) + o(\delta_n^2).
\]
Now recall that $\Delta(n)=(\Delta_i(n),i=1,\ldots,d)^T$. Thus,
\[
\Delta(n)\frac{V_{\theta(n)+\delta_n\Delta(n)}(s_n)}{\delta_n} = \frac{1}{\delta_n}\Delta(n)V_{\theta(n)}(s_n)\]\[
+\Delta(n)\Delta(n)^T \nabla V_{\theta(n)}(s_n)\]
\[
+ \frac{\delta_n}{2} \Delta(n) \Delta(n)^T \nabla^2 V_{\theta(n)}(s_n) \Delta(n) + o(\delta_n).
\]
%
Now observe from the properties of $\Delta_i(n), \forall i,n,$ that \\
(i) $E[\Delta(n)] =0$ (the zero-vector), $\forall n$, since $\Delta_i(n)\sim N(0,1)$, $\forall i,n$.\\
(ii) $E[\Delta(n)\Delta(n)^T] = I$ (the identity matrix), $\forall n$.\\
(iii) ${\displaystyle E\left[\sum_{i,j,k=1}^{d} \Delta_i(n)\Delta_j(n)\Delta_k(n)\right] =0}$. \\
Property (iii) follows from the facts that (a) $E[\Delta_i(n)\Delta_j(n)\Delta_k(n)]=0$, $\forall i\not=j\not=k$, (b) $E[\Delta_i(n)\Delta_j^2(n)] =0$, $\forall i\not=j$ (this pertains to the case where $i\not=j$ but $j=k$ above) and (c) $E[\Delta_i^3(n)] =0$ (for the case when $i=j=k$ above). These properties follow from the independence of the random variables $\Delta_i(n)$, $i=1,\ldots,d$ and $n\geq 0$, as well as the fact that they are all distributed $N(0,1)$. 
%
The claim now follows from (i)-(iii) above.
\end{proof}

In the light of Proposition~\ref{propest}, we can rewrite (\ref{reinforce}) as follows:
\[
    \theta(n+1) = \Gamma(\theta(n) - a(n)(\sum_{s}\nu(s)\nabla V_{\theta(n)}(s) + M_{n+1}\]
    \begin{equation}
        \label{equivform}
    +\beta(n))),
\end{equation}
where 
${\displaystyle \beta_i(n) = E\left[\Delta_i(n)\frac{G_n}{\delta} \mid \mathcal{F}_n \right]}$ $-\sum_{s} \nu(s) \nabla_i V_{\theta(n)}(s)$ and $\beta(n) = (\beta_1(n),\ldots,\beta_d(n))^T$. 
From Proposition~\ref{propest}, it follows that $\beta(n) = o(\delta_n)$. 

\begin{lemma}
\label{lipschitzv}
The function $\nabla V_\theta(s)$ is Lipschitz continuous in $\theta$. Further, $\exists$ a constant $K_1>0$ such that $\parallel \nabla V_\theta(s)\parallel \leq K_1(1+\parallel \theta\parallel)$.
\end{lemma}

\begin{proof}
It can be seen from (\ref{pgt}) that $V_\theta(s)$ is continuously differentiable in $\theta$. It can also be shown as in Theorem 3 of \cite{FurmstonLeverBarber} that $\nabla^2 V_\theta(s)$ exists and is continuous. Since $\theta$ takes values in $C$, a compact set, it follows that $\nabla^2 V_\theta(s)$ is bounded and thus $\nabla V_\theta(s)$ is Lipschitz continuous.

Finally, let $L_1^s>0$ denote the Lipschitz constant for the function $\nabla V_\theta(s)$. Then, for a given $\theta_0\in C$,
\[ \parallel \nabla V_\theta(s)\parallel - \parallel \nabla V_{\theta_0}(s)\parallel
\leq \parallel \nabla V_\theta(s)-\nabla V_{\theta_0}(s) \parallel
\]
\[ \leq L_1^s \parallel \theta-\theta_0\parallel
\]
\[ \leq L_1^s \parallel \theta \parallel + L_1^s\parallel \theta_0\parallel.
\]
Thus,
\( \parallel \nabla V_\theta(s)\parallel \leq \parallel \nabla V_{\theta_0}(s)\parallel + L^s_1 \parallel \theta_0\parallel + L^s_1\parallel \theta\parallel.
\)
Let $K_s \stackrel{\triangle}{=} \| \nabla V_{\theta_0}(s)\| + L^s_1\| \theta_0\|$ and $K_1 \stackrel{\triangle}{=} \max(K_s, L^s_1, s\in S)$.
Thus, $\parallel \nabla V_\theta(s)\parallel \leq K_1(1+\parallel \theta\parallel)$. Note here that since $|S|<\infty$, $K_1<\infty$ as well. The claim follows. 
\end{proof}

\begin{lemma}
\label{sqintegmg}
The sequence $(M_n,\mathcal{F}_n)$, $n\geq 0$ satisfies
\(\displaystyle{E[\| M_{n+1}\|^2\mid \mathcal{F}_n] \leq
\frac{\hat{L}}{\delta_n^2}},
\)
for some constant $\hat{L}>0$.
\end{lemma}

\begin{proof}
Note that
\[\| M_{n+1}\|^2 = \sum_{i=1}^d (M^i_{n+1})^2
\]
\[
= \sum_{i=1}^{d}\Big(\Delta_i^2(n)\frac{(G^n)^2}{\delta_n^2} +\frac{1}{\delta_n^2} E\left[\Delta_i(n) G^n \mid \mathcal{F}_n\right]^2
\]\[- 2 \Delta_i(n) \frac{G^n}{\delta_n^2} E\left[\Delta_i(n) G^n\mid \mathcal{F}_n\right]\Big).
\]
Thus,
\[E[\| M_{n+1}\|^2\mid \mathcal{F}_n] = \frac{1}{\delta_n^2}
\sum_{i=1}^d \Big(E[\Delta_i^2(n) (G^n)^2\mid \mathcal{F}_n]\]\[
- E^2[\Delta_i(n) G^n\mid \mathcal{F}_n]\Big). \]
The claim now follows from Assumption~\ref{properpolicy} and the fact that all single-stage costs are bounded (cf.~pp.174, Chapter 3 of \cite{bertsekas2}).
\end{proof}

Define now a sequence $Z_n,n\geq 0$ according to
\[ Z_n = \sum_{m=0}^{n-1} a(m)M_{m+1},
\]
$n\geq 1$ with $Z_0=0$.

\begin{lemma}
\label{mg2}
$(Z_n,\mathcal{F}_n)$, $n\geq 0$ is an almost surely convergent martingale sequence.
\end{lemma}
\begin{proof}
It is easy to see that $Z_n$ is $\mathcal{F}_n$-measurable $\forall n$. Further, it is integrable for each $n$ and moreover $E[Z_{n+1}\mid \mathcal{F}_n] = Z_n$ almost surely since $(M_{n+1},\mathcal{F}_n)$, $n\geq 0$ is a martingale difference sequence by Lemma~\ref{firstlem}. It is also square integrable from Lemma~\ref{sqintegmg}. 
The quadratic variation process of this martingale will be convergent almost surely if
\begin{equation}
    \label{quadvar}
\sum_{n=0}^{\infty} E[\| Z_{n+1}-Z_n\|^2 \mid \mathcal{F}_n]
<\infty \mbox{ a.s.}
\end{equation}
Note that
\[
E[\| Z_{n+1}-Z_n\|^2 \mid \mathcal{F}_n]
= a(n)^2 E[ \| M_{n+1}\|^2 \mid \mathcal{F}_n].
\]
Thus,
\[
\sum_{n=0}^{\infty} E[\| Z_{n+1}-Z_n\|^2 \mid \mathcal{F}_n]
= \sum_{n=0}^{\infty} a(n)^2 E[ \| M_{n+1}\|^2 \mid \mathcal{F}_n]
\]
\[
\leq \hat{L} \sum_{n=0}^{\infty} \left(\frac{a(n)}{\delta_n}\right)^2,
\]
by Lemma~\ref{sqintegmg}. (\ref{quadvar}) now follows as a consequence of Assumption~\ref{assum:ss}. Now $(Z_n,\mathcal{F}_n)$, $n\geq 0$ can be seen to be convergent from the martingale convergence theorem for square integrable martingales \cite{borkar-prob}.
\end{proof}

Consider now the following ODE:
\begin{equation}
\label{appendix:eq:kushcla-ode2}
\dot{\theta}(t) = \bar{\Gamma}(-\sum_{s}\nu(s) \nabla V_{\theta}(s)),
\end{equation}
where $\bar{\Gamma}: {\cal C}(C)\rightarrow {\cal C}({\cal R}^d)$ is defined according to
\begin{equation}
	\label{appendix:eq:projected-ode}
	\bar{\Gamma}(v(x)) = \lim_{\eta\rightarrow 0} \left(\frac{\Gamma(x+\eta v(x))-x}{\eta}\right).
 \end{equation}

Let $H\stackrel{\triangle}{=} \{\theta\mid \bar{\Gamma}(-\sum_{s}\nu(s) \nabla V_{\theta}(s))=0\}$ denote the set of all equilibria of (\ref{appendix:eq:kushcla-ode2}). By Lemma 11.1 of \cite{borkar}, the only possible $\omega$-limit sets that can occur as invariant sets for the ODE (\ref{appendix:eq:kushcla-ode2}) are subsets of $H$. 
Let $\bar{H}\subset H$ be the set of all internally chain recurrent points of the ODE (\ref{appendix:eq:kushcla-ode2}). 
Our main result below is based on 
Theorem 5.3.1 of \cite{kushcla} for projected stochastic approximation algorithms. Before we proceed further, we recall that result below.

Let $C\subset {\cal R}^d$ be a compact and convex set as before and $\Gamma:{\cal R}^d\rightarrow C$
denote the projection operator that projects any $x=(x_1,\ldots,x_d)^T \in {\cal R}^d$ to its nearest point
in $C$. 

Consider now the following
the $d$-dimensional stochastic recursion
\begin{equation}
\label{appendix:eq:kush-cla}
X_{n+1} = \Gamma(X_{n} + a(n)(h(X_n) + \xi_n + \beta_n)), 
\end{equation}
under the assumptions listed below. Also, consider the following
ODE associated with (\ref{appendix:eq:kush-cla}):
\begin{equation}
\label{appendix:eq:kushcla-ode}
\dot{X}(t) = \bar{\Gamma}(h(X(t))).
\end{equation}

Let ${\cal C}(C)$ denote the space of all continuous functions from $C$ to ${\cal R}^d$.
The operator $\bar{\Gamma}: {\cal C}(C)\rightarrow {\cal C}({\cal R}^d)$ is defined
according to
\begin{equation}
\label{appendix:eq:projected-ode}
\bar{\Gamma}(v(x)) = \lim_{\eta\rightarrow 0} \left(\frac{\Gamma(x+\eta v(x))-x}{\eta}\right),
\end{equation}
for any continuous $v:C\rightarrow {\cal R}^d$. The limit in (\ref{appendix:eq:projected-ode})
exists and is unique since $C$ is a convex set. In case this limit is not unique, one may consider the
set of all limit points of (\ref{appendix:eq:projected-ode}). Note also that from
its definition, $\bar{\Gamma}(v(x)) = v(x)$ if $x\in C^o$ (the interior of $C$). This
is because for such an $x$, one can find $\eta>0$ sufficiently small so that
$x+\eta v(x) \in C^o$ as well and hence $\Gamma(x+\eta v(x)) = x+\eta v(x)$.
On the other hand, if $x\in \partial C$ (the boundary of $C$) is such that
$x+\eta v(x) \not\in C$, for any small $\eta>0$, then $\bar{\Gamma}(v(x))$
is the projection of $v(x)$ to the
tangent space of $\partial C$ at $x$. 

Consider now the assumptions listed below.

\begin{itemize}
\item[(B1)]
The function $h:{\cal R}^d\rightarrow {\cal R}^d$ is continuous.
\item[(B2)]
The step-sizes $a(n),n\geq 0$ satisfy
\[a(n)>0 \forall n, \mbox{ } \sum_n a(n)=\infty, \mbox{ } a(n)\rightarrow 0 \mbox{ as }n\rightarrow\infty.\]
\item[(B3)] 
The sequence $\beta_n,n\geq 0$ is a bounded random sequence with
$\beta_n \rightarrow 0$ almost surely as $n\rightarrow \infty$.
\item[(B4)]
There exists $T>0$ such that $\forall \epsilon>0$,
\[ \lim_{n\rightarrow\infty} P\left( \sup_{j\geq n} \max_{t\leq T} \left|
\sum_{i=m(jT)}^{m(jT+t)-1} a(i)\xi_i\right| \geq \epsilon \right) = 0. \]
\item[(B5)]
The ODE (\ref{appendix:eq:kushcla-ode}) has a compact subset $K$ of ${\cal R}^N$ as its set of
asymptotically stable equilibrium points.
\end{itemize}

Let $t(n),n\geq 0$ be a sequence of positive real numbers
defined according to $t(0)=0$ and for $n\geq 1$,
${\displaystyle t(n) = \sum_{j=0}^{n-1} a(j)}$.
By Assumption~(B2), $t(n)\rightarrow\infty$ as $n\rightarrow\infty$.
Let ${\displaystyle m(t) = \max\{n\mid t(n) \leq t\}}$. Thus, $m(t)\rightarrow\infty$
as $t\rightarrow\infty$.
Assumptions (B1)-(B3) correspond to A5.1.3-A5.1.5 of \cite{kushcla} while (B4)-(B5) correspond to A5.3.1-A5.3.2 there.

\cite[Theorem 5.3.1 (pp.~191-196)]{kushcla} essentially says the following:

\begin{theorem}[Kushner and Clark Theorem]
\label{appE:theorem1}
Under Assumptions~(B1)--(B5),
almost surely, $X_n\rightarrow K$ as $n\rightarrow\infty$.
\end{theorem}\index{theorem, Kushner and Clark}

\begin{theorem}
\label{theorem-conv-proj}
The iterates $\theta(n), n\geq 0$ governed by (\ref{reinforce}) converge almost surely to $\bar{H}$.
\end{theorem}

\begin{proof}
In lieu of the foregoing, we rewrite (\ref{reinforce}) according to
\[
    \theta_i(n+1) = \Gamma_i\Big(\theta_i(n) -a(n)\sum_s \nu(s)\nabla_i V_{\theta(n)}(s)\]
\begin{equation}
    \label{reinforce2}
    -a(n) \beta_i(n) 
    + M^i_{n+1}\Big),
    \end{equation}
    where $\beta_i(n)$ is as in (\ref{equivform}).  
    We shall proceed by verifying Assumptions~(B1)-(B5) and subsequently appeal to Theorem 5.3.1 of \cite{kushcla} (i.e., Theorem 1 above) to claim convergence of the scheme. 
Note that Lemma~\ref{lipschitzv} ensures Lipschitz continuity of $\nabla V_\theta(s)$ implying (B1). Next, from (B2), since $\delta_n\rightarrow 0$, it follows that $a(n)\rightarrow 0$ as $n\rightarrow\infty$. Thus, Assumption (B2) holds as well.
Now from Lemma~\ref{lipschitzv}, it follows that $\sum_s \nu(s) \nabla V_\theta(s)$ is uniformly bounded since $\theta\in C$, a compact set. Assumption~(B3) is now verified from Proposition~\ref{propest}. Since $C$ is a convex and compact set, Assumption (B4) holds trivially. Finally, Assumption (B5) is also easy to see as a consequence of Lemma~\ref{mg2}. Now note that for the ODE (\ref{appendix:eq:kushcla-ode2}), $F(\theta)=\sum_s \nu(s) V_\theta(s)$ serves as an associated Lyapunov function and in fact
\[
\nabla F(\theta)^T \bar{\Gamma}(-\sum_{s}\nu(s) \nabla V_{\theta}(s)) \]
\[= (\sum_s \nu(s) \nabla_\theta V_\theta(s))^T\bar{\Gamma}(-\sum_{s}\nu(s) \nabla V_{\theta}(s))
\leq 0.
\]
For $\theta\in C^o$ (the interior of $C$), it is easy to see that $\bar{\Gamma}(-\sum_s \nu(s) \nabla V_\theta(s))$ $=-\sum_s \nu(s)\nabla V_\theta(s)$, and
\begin{eqnarray*}
\nabla F(\theta)^T \bar{\Gamma}(-\sum_{s}\nu(s) \nabla V_{\theta}(s)) &<&
0 \mbox{ if } \theta \in H^c \cap C \\
&=& 0 \mbox{ o.w.}
\end{eqnarray*}
For $\theta\in \delta C$ (the boundary of $C$), there can additionally be spurious attractors, see \cite{kushyin}, that are also contained in $H$.
The claim now follows from Theorem~5.3.1 of \cite{kushcla}.
\end{proof}

\section{Conclusions}
\label{conclusions}

We presented a version of the Reinforce algorithm for the setting of episodic tasks that incorporates a one-simulation SF algorithm and proved its convergence using the ODE approach. 
In a longer version of this paper, we shall consider also the average cost case and present an algorithm based on \cite{bhatbor3} that does not rely on regeneration epochs for performing updates. Moreover, we shall show detailed empirical results comparing our proposed algorithms with the policy gradient scheme and other algorithms. 

\bibliographystyle{plain}
\bibliography{references}

\begin{thebibliography}{10}

\bibitem{bertsekas2}
D.~P. Bertsekas.
\newblock {\em Dynamic Programming and Optimal Control, Vol.II}.
\newblock Athena Scientific, 2012.

\bibitem{bhat1}
S.~Bhatnagar.
\newblock Adaptive multivariate three-timescale stochastic approximation
  algorithms for simulation based optimization.
\newblock {\em ACM Transactions on Modeling and Computer Simulation},
  15(1):74--107, 2005.

\bibitem{bhat2}
S.~Bhatnagar.
\newblock Adaptive {Newton}-based smoothed functional algorithms for simulation
  optimization.
\newblock {\em ACM Transactions on Modeling and Computer Simulation},
  18(1):2:1--2:35, 2007.

\bibitem{bhatbor3}
S.~Bhatnagar and V.S. Borkar.
\newblock Multiscale chaotic {SPSA} and smoothed functional algorithms for
  simulation optimization.
\newblock {\em Simulation}, 79(10):568--580, 2003.

\bibitem{bhatnagar-book}
S~Bhatnagar, H.~L. Prasad, and L.~A. Prashanth.
\newblock {\em Stochastic Recursive Algorithms for Optimization: Simultaneous
  Perturbation Methods (Lecture Notes in Control and Information Sciences)},
  volume 434.
\newblock Springer, 2013.

\bibitem{bhatnagar2007incremental}
S.~Bhatnagar, R.S. Sutton, M.~Ghavamzadeh, and M.~Lee.
\newblock Incremental natural actor-critic algorithms.
\newblock {\em Advances in neural information processing systems}, 20, 2007.

\bibitem{bhatnagar2009natural}
S.~Bhatnagar, R.S. Sutton, M.~Ghavamzadeh, and M.~Lee.
\newblock Natural actor--critic algorithms.
\newblock {\em Automatica}, 45(11):2471--2482, 2009.

\bibitem{borkar-prob}
V.~S. Borkar.
\newblock {\em Probability Theory: An Advanced Course}.
\newblock Springer, New York, 1995.

\bibitem{borkar}
V.~S. Borkar.
\newblock {\em Stochastic Approximation: A Dynamical Systems Viewpoint, 2'nd
  Edition}.
\newblock Cambridge University Press, 2022.

\bibitem{cao}
X.-R. Cao.
\newblock {\em Stochastic Learning and Optimization: A Sensitivity-Based
  Approach}.
\newblock Springer, 2007.

\bibitem{choram1}
E.~K.~P. Chong and P.~J. Ramadge.
\newblock Optimization of queues using an infinitesimal perturbation
  analysis-based stochastic algorithm with general update times.
\newblock {\em SIAM J. Cont. and Optim.}, 31(3):698--732, 1993.

\bibitem{FurmstonLeverBarber}
T.~Furmston, G.~Lever, and D.~Barber.
\newblock Approximate {Newton} methods for approximate policy search in
  {Markov} decision processes.
\newblock {\em Journal of Machine Learning Research}, 17:1--51, 2016.

\bibitem{hocao}
Y.~C. Ho and X.~R. Cao.
\newblock {\em Perturbation Analysis of Discrete Event Dynamical Systems}.
\newblock Kluwer, Boston, 1991.

\bibitem{katkul}
V.~Ya Katkovnik and Yu~Kulchitsky.
\newblock Convergence of a class of random search algorithms.
\newblock {\em Automation Remote Control}, 8:1321--1326, 1972.

\bibitem{kw}
J.~Kiefer and J.~Wolfowitz.
\newblock Stochastic estimation of the maximum of a regression function.
\newblock {\em Ann. Math. Statist.}, 23:462--466, 1952.

\bibitem{konda1999actor}
V.R. Konda and V.S. Borkar.
\newblock Actor-critic--type learning algorithms for markov decision processes.
\newblock {\em SIAM Journal on control and Optimization}, 38(1):94--123, 1999.

\bibitem{konda2003onactor}
V.R. Konda and J.N. Tsitsiklis.
\newblock Onactor-critic algorithms.
\newblock {\em SIAM journal on Control and Optimization}, 42(4):1143--1166,
  2003.

\bibitem{kushcla}
H.~J. Kushner and D.~S. Clark.
\newblock {\em Stochastic Approximation Methods for Constrained and
  Unconstrained Systems}.
\newblock Springer Verlag, New York, 1978.

\bibitem{kushyin}
H.~J. Kushner and G.~G. Yin.
\newblock {\em Stochastic Approximation Algorithms and Applications}.
\newblock Springer Verlag, New York, 1997.

\bibitem{marbach2001simulation}
P.~Marbach and J.N. Tsitsiklis.
\newblock Simulation-based optimization of {Markov} reward processes.
\newblock {\em IEEE Transactions on Automatic Control}, 46(2):191--209, 2001.

\bibitem{prashanth2017rdsa}
L~A Prashanth, S.~Bhatnagar, M.C. Fu, and S.I. Marcus.
\newblock Adaptive system optimization using random directions stochastic
  approximation.
\newblock {\em IEEE Transactions on Automatic Control}, 62(5):2223--2238, 2017.

\bibitem{rubinstein}
R.~Y. Rubinstein.
\newblock {\em Simulation and the Monte Carlo Method}.
\newblock Wiley, New York, 1981.

\bibitem{spall1992multivariate}
J.C. Spall.
\newblock Multivariate stochastic approximation using a simultaneous
  perturbation gradient approximation.
\newblock {\em IEEE Transactions on Automatic Control}, 37(3):332--341, 1992.

\bibitem{spall1997one}
J.C. Spall.
\newblock A one-measurement form of simultaneous perturbation stochastic
  approximation.
\newblock {\em Automatica}, 33(1):109--112, 1997.

\bibitem{suttonbarto}
R.~S. Sutton and A.~W. Barto.
\newblock {\em Reinforcement Learning, 2'nd Edition}.
\newblock MIT Press, 2018.

\bibitem{sutton1999policy}
R.~S. Sutton, D.~A. McAllester, S.~P. Singh, and Y.~Mansour.
\newblock Policy gradient methods for reinforcement learning with function
  approximation.
\newblock In {\em Advances in Neural Information Processing Systems},
  volume~99, pages 1057--1063, 1999.

\bibitem{williams1992simple}
R.J. Williams.
\newblock Simple statistical gradient-following algorithms for connectionist
  reinforcement learning.
\newblock {\em Reinforcement learning}, pages 5--32, 1992.

\end{thebibliography}
\end{document}